\documentclass[nohyperref]{article}

\usepackage{color, colortbl}
\definecolor{Gray}{gray}{0.9}

\usepackage{colonequals}
\usepackage{microtype}
\usepackage{graphicx}
\usepackage{subfigure}
\usepackage{booktabs} 
\usepackage[section]{placeins}
\usepackage{hyperref}

\usepackage[accepted]{icml2023}

\usepackage{amsmath}
\usepackage{amssymb}
\usepackage{mathtools}
\usepackage{amsthm}

\usepackage[capitalize,noabbrev]{cleveref}

\theoremstyle{plain}
\newtheorem{theorem}{Theorem}[section]
\newtheorem{proposition}[theorem]{Proposition}

\theoremstyle{definition}

\newtheorem{assumption}[theorem]{Assumption}
\theoremstyle{remark}

\usepackage[textsize=tiny]{todonotes}
\usepackage{xspace}

\newcommand{\green}[1]{{\color{violet}{({#1})}}}
\newcommand{\citeremind}[1]{{[$\blacksquare$}]}

\newcommand{\ourmodel}{Multi-Symmetry Ensemble\xspace}
\newcommand{\ourmodels}{Multi-Symmetry Ensembles\xspace}
\newcommand{\mse}{MSE\xspace}
\newcommand{\rebutadd}[1]{{#1}}

\icmltitlerunning{Multi-Symmetry Ensembles: Improving Diversity and Generalization via Opposing Symmetries}

\begin{document}

\twocolumn[
\icmltitle{Multi-Symmetry Ensembles: Improving Diversity and Generalization \\ via Opposing Symmetries}




\begin{icmlauthorlist}
\icmlauthor{Charlotte Loh}{eecs,mitibm}
\icmlauthor{Seungwook Han}{eecs}
\icmlauthor{Shivchander Sudalairaj}{mitibm}
\icmlauthor{Rumen Dangovski}{eecs}
\icmlauthor{Kai Xu}{amazon}
\icmlauthor{Florian Wenzel}{aws}
\icmlauthor{Marin Solja\v{c}i\'{c}}{physics}
\icmlauthor{Akash Srivastava}{mitibm}
\end{icmlauthorlist}

\icmlaffiliation{eecs}{MIT EECS}
\icmlaffiliation{physics}{MIT Physics}
\icmlaffiliation{mitibm}{MIT-IBM Watson AI Lab}
\icmlaffiliation{aws}{AWS (work done outside of Amazon)}
\icmlaffiliation{amazon}{Amazon (work done outside of Amazon)}

\icmlcorrespondingauthor{Charlotte Loh}{cloh@mit.edu}

\icmlkeywords{Machine Learning, ICML}

\vskip 0.3in
]



\printAffiliationsAndNotice{}  

\begin{abstract}
Deep ensembles (DE) have been successful in improving model performance by learning diverse members via the stochasticity of random initialization.
While recent works have attempted to promote further diversity in DE via hyperparameters or regularizing loss functions, these methods primarily still rely on a stochastic approach to explore the hypothesis space. 
In this work, we present Multi-Symmetry Ensembles (MSE), a framework for constructing diverse ensembles by capturing the multiplicity of hypotheses along symmetry axes, which explore the hypothesis space beyond stochastic perturbations of model weights and hyperparameters.
We leverage recent advances in contrastive representation learning to create models that separately capture opposing hypotheses of invariant and equivariant \rebutadd{functional classes} and present a simple ensembling approach to efficiently combine appropriate hypotheses for a given task.
We show that MSE effectively captures the multiplicity of conflicting hypotheses that is often required in large, diverse datasets like ImageNet.
As a result of their inherent diversity, MSE improves classification performance, uncertainty quantification, and generalization across a series of transfer tasks. 
Our code is available at  \url{https://github.com/clott3/multi-sym-ensem}

\end{abstract}




\section{Introduction}
\label{sec:introduction}
\begin{figure}[t]
\vskip 0.1in
\begin{center}
\centerline{\includegraphics[width=1.25\columnwidth]{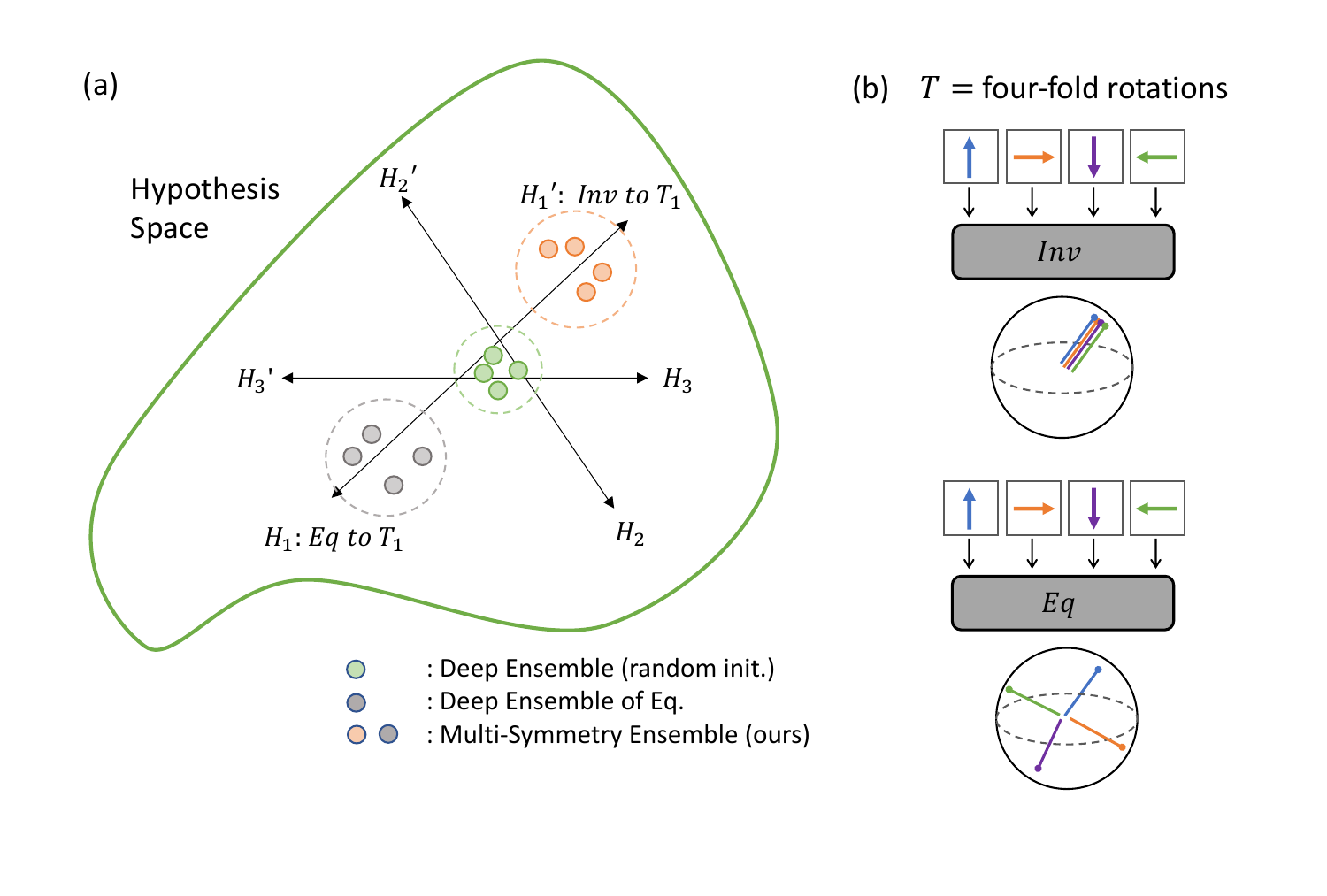}}
\vskip -0.32in
\caption{\textbf{(a)} A comparative illustration of the diversity in the hypothesis space that traditional deep ensembles and our \ourmodels can achieve. While deep ensembles are effective at capturing different solutions around one hypothesis, \ourmodels can learn diverse solutions around inherently opposing hypotheses. \textbf{(b)} Schematic visualization of invariance (top) v.s. equivariance (bottom) for the four-fold rotation. The spheres denote the representation space of the models.}
\vskip -0.5in
\label{fig:illus}
\end{center}
\end{figure}
The field of computer vision has seen significant progress in various tasks such as classification and semantic segmentation in recent years. 
This success can be attributed to the advancements in model architectures, learning methods, and the availability of large-scale datasets \cite{vit, jft, simclr}. 
Large and diverse datasets have proved crucial in improving performance, yet they present new challenges. The increased diversity of datasets makes it more difficult for a single dominant hypothesis to capture all semantic classes. 
To overcome this problem, model ensembling~\citep{nn_ensem,bagging} can be utilized to combine multiple networks.
A popular approach is Deep Ensembles (DE)~\citep{DE}, which combines networks with different random initializations and relies on the non-convexity of the loss landscape~\citep{DE_losslandscape} and stochasticity of the training algorithm to arrive at different solutions.
They often significantly improve model performance and uncertainty quantification~\citep{uq_ovadia}.\looseness=-1

Their success can be attributed to the diversity amongst the ensemble members~\citep{dice}; ensemble performance can be significantly improved relative to the individual models when the members are diverse and their errors are uncorrelated (i.e. when the members make mistakes on different samples). 
However, purely relying on the stochasticity in the random initialization and the training algorithm can only provide a limited amount of diversity~\cite{dice} and previous works have attempted to promote diversity further by training models with different data augmentations~\cite{divensem_improve_calib}, hyperparameters~\cite{hyperparameter_ens}, or explicitly encouraged via loss functions~\cite{pang_diversity,dice}. 
Nonetheless, these methods primarily still rely on a stochastic approach to explore the hypothesis space.

In this work, we present a framework for constructing ensembles that are inherently diverse with respect to certain symmetry groups and thus in this regard,
are \textit{non-stochastic in exploring the hypothesis space}.
We argue that current ensembling approaches are not effective in capturing the multiplicity of hypotheses, particularly along symmetry axes, which are necessary for large vision datasets.
We motivate this with an intuitive example of rotational symmetry on the ImageNet~\citep{deng2009imagenet} dataset.
Recent works~\citep{rotnet,essl} have demonstrated the effectiveness of encoding rotational equivariance
\footnote{Equivariance can be best understood when contrasted with invariance -- while invariance requires the outputs to be unchanged when the inputs are transformed, equivariance requires the outputs to transform \textit{according to the way inputs are transformed}. \rebutadd{While invariance is a trivial instance of equivariance (where $T_g'$ in~\cref{eq:equi_def} is the identity), in this work we use ``equivariance'' to refer specifically to non-trivial equivariances.}} on ImageNet.
\rebutadd{In this work, the term ``equivariance'' is used to explicitly refer to non-trivial equivariances, i.e. not encompassing the trivial instance, invariance (see footnote 1).}
\rebutadd{Empirically, we found equivariance to be useful in images with a clear stance (e.g. dogs, where an upside-down dog is never observed in the dataset) and thus encoding information about its pose (i.e. rotation) aids their characterization.}
However, in large datasets like ImageNet, there also exist images like flowers that contain rotational symmetry and thus encoding rotational invariance, \rebutadd{i.e. the removal of pose information}, may be more desirable (see~\cref{fig:illus}b for an illustration of invariance versus equivariance).

Given the opposing nature of these hypotheses (see \rebutadd{footnote 1} and~\cref{fig:illus}b), a stochastic ensembling approach cannot capture both simultaneously; i.e. a deep ensemble of rotational equivariant classifiers cannot be made rotational invariant by simply perturbing hyperparameters or model weights at initialization. 
We visually illustrate this point in Figure \ref{fig:illus}a. 
To address this problem, we leverage recent advances in contrastive representation learning~\cite{simclr, essl} to create models that separately capture opposing invariant and equivariant hypotheses  around a given symmetry group. 
In particular, in contrast to task-specific diversity promoting mechanisms of previous works~\cite{pang_diversity,dice}, our approach aims to learn diverse representations that individually respect different symmetries and such task-agnosticity is desirable when transferring to new downstream tasks.

We present a practical, greedy ensembling approach that efficiently combines appropriate hypotheses for a given set of tasks. 
We provide extensive empirical results and analyses to demonstrate the superior performance of our method in classification performance, uncertainty quantification and transfer learning on new datasets. 
Our contributions can be summarized as follows:
\begin{itemize}
    \item We empirically show that large, diverse datasets like ImageNet inherently have multiple and conflicting dominant hypotheses for classification. 
    \item We propose \ourmodel (\mse), an ensembling method to train and combine models of opposing hypotheses with respect to certain symmetry groups. In contrast to previous works that rely on stochasticity created via random initializations or hyperparameters, we directly guide diversity exploration along the axes of symmetry.
    \item We demonstrate that \mse  can leverage weaker models from the opposing hypothesis that improve performance more than the ensemble of higher-accuracy models corresponding to the leading hypothesis. To this end, we conduct a detailed empirical study to show that \mse improves classification performance and uncertainty quantification, and better generalizes across a series of transfer tasks.  
    \item We also show that our method applies to different symmetry groups and that opposing hypotheses across multiple axes of symmetries further improve diversity.
\end{itemize}

\section{Background and Related Work}
\label{sec:related_work}

\paragraph{Neural network ensembles and diversity.}
Using an ensemble of neural networks to improve performance and generalization is a well known technique in machine learning that existed decades ago~\citep{nn_ensem, bagging}. 
Deep ensembles~\cite{DE} create an ensemble of networks by using different random initializations, and ~\citet{mcdropout, batchensemble, mimo} improve upon this by making it more computationally efficient. Diversity is an important feature in ensembles, since averaging many models that give the exact same prediction is no better than using a single model. 
\citet{pang_diversity, lee_diversity, diversity_prediction} create diversity by changing the losses or the architecture. \citet{hyperparameter_ens} create ensembles using different hyperparameters and~\citet{divensem_improve_calib, augmix} leverage data augmentation strategies. All of these methods rely on the stochasticity from the architecture, random initialization, or hyperparameters to generate different solutions. However, our work differs in that, we learn diverse solutions by leveraging opposing functional classes along certain symmetry groups. Moreover, the supervised learning settings do not focus on the transferability of representations, and therefore we propose an ensembling method that transfers better to a series of new downstream datasets.
Along the line of learning diverse representations, \citet{no_one_rep_rules, modelsoups} both conducted large-scale empirical studies of ensembling representations and models across architectures, training methods and datasets.
Our work differs from these in that instead of a large-scale study of ensembling representations, our work introduces and focuses on a new technique of creating diversity in representations by using equivariances and invariances.\looseness=-1
\paragraph{Contrastive Learning and augmentations.}
Contrastive representation learning~\cite{moco,simclr} is an effective method for learning transferable representations with self-supervised learning.  
The role of augmentations in contrastive learning has been extensively studied~\cite{simclr,tian2020makes,looc,reed2021selfaug,essl} with the objective of discovering useful augmentations to improve performance on downstream tasks. In contrast, our work takes a general approach of creating more robust classifiers by ensembling models of opposing equivariances. 
\citet{looc} designed a training objective that simultaneously computes a contrastive loss on a variety of projected representations, and each loss is associated with leaving out one augmentation from the complete set of augmentations. 
Our contribution is different from~\cite{looc}, because we use equivariance, instead of removing augmentations. 
In addition, rather than a joint or concatenated latent space that is specialized to the removed augmentation, we create independent latent spaces and use an ensembling approach to accumulate the predictions of each member.
A growing field of contributions introduce equivariance to models via self-supervised learning~\cite{essl,equimod}.\looseness=-1
\paragraph{Equivariant neural networks.}
Let $f$ be a continuous function (parameterized with an encoder network) and $\mathbf{x}$ be the input; equivariance to a group $G$ of transformations is mathematically defined as
\begin{equation}\label{eq:equi_def}
    \forall{\mathbf{x}}: f(T_g(\mathbf{x})) = T_g'(f(\mathbf{x}))
\end{equation}
where $T_g$ denotes the transformation associated with a group element $g\in G$. 
In this formulation, invariance can be understood as a particular (trivial) instance where $T_g'$ is the identity function, i.e. $f(T_g(\mathbf{x})) = f(\mathbf{x})$ and the output of the network does not change after a transformation to the input.
Instead, equivariance requires the network output to change in a well-defined manner according to the way the input has been transformed.
Intuitively, the difference between invariance and non-trivial equivariance can be understood as follows; while invariance encourages representations to remove information about the way they are transformed, non-trivial equivariance encourages the network to preserve this transformation information.
This allows for a broader class of inductive biases that allows the model to make a decision on how to utilize this information during prediction.

Group equivariant neural networks~\cite{cohen2016group,weiler_general_2019,weiler_3d_2018} are usually designed by generalizing convolutional neural networks to arbitrary groups by constructing specialized kernels that satisfies the equivariance constraints. 
As such these networks usually require specialized architectures that are less commonly used on large-scale vision benchmarks.
\citet{essl} proposes a technique to encourage equivariance by using a prediction loss and show that approximate equivariance can be achieved by predicting the transformation.
Along a similar line of work,~\cite{colorful, rotnet,jigsaw} propose to learn visual representations by pretext tasks of predicting transformations.
In our work, to avoid having specialized architectures and to keep the framework highly general and flexible, we adopt the method proposed in~\cite{essl} to achieve \textit{approximate equivariance} by a training objective that predicts the transformations applied to the input during the self-supervised learning stage. 
However, instead of learning better representations, our work focuses on the importance of creating ensembles containing members having opposing equivariances. 
\rebutadd{Furthermore,~\cite{essl} showed that rotational equivariance leads to better representations while rotational invariance is harmful; in this work, we show that while equivariance is useful for the majority of classes, there is a significant proportion of the data that can benefit from rotation invariance.}
\section{\ourmodels}
\label{sec:equivariant_ensembles}
We go beyond the typical deep ensembling approach by constructing ensembles that include opposing hypotheses along a set of symmetries.
 We start by pre-training representation learning models using contrastive learning methods \cite{simclr,essl}. 
The pre-training step allows for encoding the necessary equivariances and invariances into the models. 
During fine-tuning, the pre-trained models are adapted into classifiers, and finally, these classifiers are combined into an ensemble.
We demonstrate analytically in a simple setting via~\cref{prop:simple_analytical} in Appendix~\ref{app:formalism} that the trained classifiers of the equivariant and invariant models capture different hypotheses.





\subsection{Invariant and Equivariant Constrastive Learners}
We now describe the paradigm to obtain the diverse ensemble members by inducing different equivariance and invariance constraints to the models.
For ensemble member $m$, let $f_m(\cdot, \theta_m)$ denote the backbone encoder and $p_m(\cdot, \phi_m)$ the projector \rebutadd{(here, a 3-layer MLP), parameterized by $\theta_m$ and $\phi_m$ respectively}. Let $T_{base}$ be the base set of transformations (e.g., RandomResizedCrop, ColorJitter).
We realize the axis of symmetry through the transformations in SSL. 
Let $T^m$ denote the transformation to which member $m$ should be invariant or equivariant.

\rebutadd{Contrastive learning operates by learning representations such that views of an image created via $T_{base}$ are pulled closer together while pushed away from other images. In doing so, the model learns representations that are invariant to $T_{base}$. This is realized through the InfoNCE loss \cite{simclr}. 
Specifically, for a batch of $B$ samples, the loss is
\begin{equation}
     \mathcal{L}_{CL}^m = \sum_{i=1}^B -\log \frac{\exp(\hat{\mathbf{z}}_i^m \cdot \hat{\mathbf{z}}_j^m) / \tau }
     {\sum_{k\neq i}\exp(\hat{\mathbf{z}}_i^m \cdot \hat{\mathbf{z}}_k^m / \tau) }
\end{equation}
where $\hat{\mathbf{z}_i^m}$ and $\hat{\mathbf{z}_j^m}$ are the $\ell_2$-normalized representations of two views of an input $\mathbf{x}_i$ and $\hat{\mathbf{z}_i^m}= p_m\circ f_m(\mathbf{x}_i)/ ||p_m \circ f_m(\mathbf{x}_i)||$,  and $\tau$ is a temperature hyperparameter.
}
\paragraph{Learning invariant models.}
 Leveraging the contrastive learning framework, we learn an invariant model \rebutadd{by adding $T_m$ into the set of transformations, i.e.}
by optimizing the InfoNCE loss \cite{simclr} with the augmentations set to $T=T_{base} \cup \{T_m\}$. 
\paragraph{Learning equivariant models.}
We learn a model that is equivariant to $T^m$ by initializing a separate prediction network $h_m(\cdot, \psi_m)$ and use a prediction loss as proposed in~\cite{essl}. Let $G^m$ be a group to which member $m$ is equivariant, i.e. its elements $g \in G^m$ transform the inputs/outputs according  to~\cref{eq:equi_def}.  
The goal of $\mathcal{L}_{eq}^m$ is for the model to predict $g$ from the representation $h_m\circ f_m(T_g (\mathbf{x}_i))$.
By doing such, we encourage equivariance to $G^m$. In our work, we consider discrete and finite groups of image transformations (e.g., 4-fold rotations, color inversion (2-fold), and half-swaps (2-fold)). For discrete groups, $\mathcal{L}_{eq}^m$ takes the form of a cross-entropy loss,
\begin{equation}\label{eq:eq_loss}
     \mathcal{L}_{eq}^m = \sum_{i=1}^B \sum_{g}^{|G|} H(h_m\circ f_m(T_g(\mathbf{x}_i)), g)
\end{equation}
where $H$ denotes the cross-entropy loss function and $|G|$ denotes the order or cardinality of the group, i.e. number of elements. 
As an example, for the group of 4-fold rotations, $g$ takes on values in $\{0,1,2,3\}$ corresponding to $T_g$ in $\{0^{\circ}, 90^{\circ}, 180^{\circ}, 270^{\circ}\}$ rotation respectively.
The sum over $g$ is explained as follows; for every input, four versions are created for each of the 4 possible rotations and a cross-entropy loss is applied with their corresponding label in $\{0,1,2,3\}$. 
The combined optimization objective of an equivariant model for a batch of $B$ samples is $\mathcal{L} = \sum_{m=1}^{M} \mathcal{L}_{CL}^m + \lambda \mathcal{L}_{eq}^m$.
\rebutadd{Here, the InfoNCE loss $\mathcal{L}_{CL}^m$ encourages invariance only to $T_{base}$, i.e. $T_m$ is not included in the set of augmentations.}
\vspace{-0.3em}
\paragraph{Forming the ensemble.}
The contrastive pretraining step ensures that the representation learners have the appropriate equivariance and invariances. The next step is to convert these pretrained models into classifiers. This can be done using two methods: linear-probing or fine-tuning. Linear-probing involves training a logistic regression model to map the learned representations to the semantic classes while keeping the pretrained models frozen. Fine-tuning, on the other hand, allows the pretrained models to be updated during training, often resulting in higher accuracies on the same dataset. 
In this work, we always use fine-tuning to convert the pretrained models to classifiers unless specified otherwise. We propose two strategies for ensembling these classifiers: (1) \textit{Random} and (2) \textit{Greedy}. 
\rebutadd{In both cases, we start by selecting a random model from the leading hypothesis and sequentially add models until the ensemble has $M$ members.} 

\rebutadd{(1) \textit{Random}: \mse under the \textit{Random} strategy alternates between the two functional classes at every stage, where a random model from that functional class is sampled without replacement, i.e. \mse always consist of models from both hypotheses.}
The baselines under the \textit{Random} strategy is equivalent to randomly selecting $M$ models.

\rebutadd{(2) \textit{Greedy}:} The \textit{Greedy} strategy is inspired by the approach of \cite{hyperparameter_ens}. 
At each stage, the best model is chosen based on the validation set score by searching over all models. 

We compute the ensemble prediction $\bar{f}(\mathbf{x})$ by taking the mean of the member's \textit{prediction probabilities}
    $\bar{f}(\mathbf{x}) = \frac{1}{M} \sum_{i=1}^M f_i (\mathbf{x})$.

\section{Experimental Setup}
\label{sec:exp_setup}
We use the standard ResNet-50 architecture for the backbone encoder and follow the experimental setup in~\cite{essl}. 
Our main results consider four-fold rotation transformation as the primary hypothesis class.
All contrastive learning models were trained for 800 epochs with a batch size of 4096.
For the equivariant models, $h_m$ is a 3-layer MLP and $\lambda$ is fixed to 0.4.
Additional training details can be found in the Appendix.
\paragraph{Evaluation Protocol.}
After contrastive pre-training, we initialized a linear layer for each backbone and fine-tuned them end-to-end for 100 epochs using the SGD optimizer with a cosine decay learning rate schedule. We conducted a grid search to optimize the learning rate hyperparameter for each downstream task.
\paragraph{Transfer tasks.}
We evaluated the transfer learning performance on 4 natural image datasets. Through these experiments, we evaluated the generalization performance of \ourmodels on new downstream tasks and to show how models with opposing hypotheses can contribute to meaningful diversity. For each dataset, we randomly initialized a linear classifier for each encoder pre-trained on ImageNet and fine-tuned both the encoder and the linear head for 100 epochs. Following the approach in \cite{transfer2019}, we performed hyperparameter tuning for each model-dataset combination and selected the best hyperparameters using a validation set. For the iNaturalist-1K dataset \cite{inat}, due to its large size and computational limitations, we used the linear evaluation protocol \cite{Wu2018UnsupervisedFL, cpc, Bachman2019LearningRB, Wenzel2022AssayingOG} which involves training a linear classifier on top of a frozen encoder.
\section{Results}
\label{sec:results}
In the following sections, we provide empirical evidence to support our claim that the diversity of opposing hypotheses along the symmetry axes improves ensemble performance, both in terms of model accuracy and generalization. We begin by demonstrating that both the invariant and equivariant hypotheses along the rotational symmetry tend to be equally dominant in large datasets like ImageNet. Next, we show that \mse, which incorporates these hypotheses, outperforms strong DE-based baselines that do not. We then provide an analysis of diversity and uncertainty quantification of \mse. In~\cref{sec:transform}, we evaluate \mse on a set of transfer tasks. Finally, we study the impact of exploring opposing hypotheses along different symmetry groups on model performance.

\begin{table}[t]
\caption{\textbf{Most suitable functional class differs within a dataset.} The top-half shows the overall accuracy for models from the SimCLR baseline and each of the opposing hypotheses wrt 4-fold rotations. The bottom-half shows the proportion of classes within each dataset where each hypotheses dominate (i.e. averaged over all samples within the class), suggesting that hypotheses apart from the one with the highest individual accuracy are still beneficial.}
\label{tab:proportion}
\begin{center}
\begin{small}
\begin{sc}
\begin{tabular}{@{}lcc@{}}
\toprule
\multicolumn{3}{l}{\textbf{Model Accuracy on ImageNet (\%)}}                  \\ \midrule
{Baseline}                     & \multicolumn{2}{c}{{76.5}} \\
Eq                     & \multicolumn{2}{c}{76.9} \\
Inv                    & \multicolumn{2}{c}{76.0} \\ \midrule
\multicolumn{3}{l}{\textbf{Proportion of Classes (\%)}}           \\ \midrule
Eq $>$ Inv    & \multicolumn{2}{c}{47.7} \\
Eq $<$ Inv       & \multicolumn{2}{c}{36.3} \\
Eq $==$ Inv              & \multicolumn{2}{c}{16.0} \\ \bottomrule
\end{tabular}
\end{sc}
\end{small}
\end{center}
\end{table}

\paragraph{Dominance of hypothesis are class-dependent.}
In Table \ref{tab:proportion}, we compare two models $f_\mathrm{roteq}$ and $f_\mathrm{rotinv}$ that respectively have trained to be invariant ($\mathrm{Inv}$) and equivariant ($\mathrm{Eq}$) to four-fold rotation as contrastive learners. Even though the invariant model falls behind quite significantly from the equivariant model by 0.9\% in the overall performance on ImageNet, in contrast to the observation from \cite{no_one_rep_rules, modelsim2019}, we found the dominance of a hypothesis to be highly class-dependent, as opposed to the leading hypothesis performing better uniformly across all classes. 
While the leading equivariant hypothesis dominates in 47.7\% of ImageNet classes, the invariant still proves to be more useful in a significant 36.3\% of the classes. We repeat this experiment for a number of large and small datasets, as shown in Figure \ref{fig:eff_bar}, and found that large datasets tend to follow this trend. 

\begin{table}[tb]
\caption{\textbf{\ourmodels capturing opposing hypothesis outperform naive ensembles of the same hypothesis.} The top-half of the table compares the acccuracy of naive ensemble of a single hypothesis and a random ensemble of both equivariant and invariant hypotheses. We show that as the number of members in the ensemble grow, capturing weaker performing models from the opposing hypothesis outperforms the naive-counterpart. The lower half of the table shows that, the gains are further amplified when the ensembles are chosen in a greedy manner.}
\label{tab:IN_results}
\setlength\tabcolsep{2.8pt}
\begin{center}
\begin{small}
\begin{sc}
\begin{tabular}{lcccccccc}
\toprule
& $M=2$ & $M=3$ & $M=4$ & $M=5$ \\
\midrule
\multicolumn{5}{l}{\textbf{Random Ensemble} } \\
\midrule
\rebutadd{Inv}  & \rebutadd{77.4{\scriptsize$\pm$0.0}} & \rebutadd{78.0{\scriptsize$\pm$0.1}} & \rebutadd{78.4{\scriptsize$\pm$0.0}} & \rebutadd{78.5{\scriptsize$\pm$0.0}} \\
Eq  & 78.2{\scriptsize$\pm$0.1} & 78.7{\scriptsize$\pm$0.1} & 78.9{\scriptsize$\pm$0.1} & 79.1{\scriptsize$\pm$0.0} \\
Eq + Inv & 78.2{\scriptsize$\pm$0.1} & 78.8{\scriptsize$\pm$0.0} & 79.1{\scriptsize$\pm$0.1} & 79.3{\scriptsize$\pm$0.1}\\
\midrule
\multicolumn{5}{l}{\textbf{Greedy Ensemble} } \\
\midrule
\rebutadd{Inv} & \rebutadd{77.65{\scriptsize$\pm$0.02}} & \rebutadd{78.24{\scriptsize$\pm$0.04}} & \rebutadd{78.59{\scriptsize$\pm$0.02}} & \rebutadd{78.75{\scriptsize$\pm$0.01}} \\
Eq   & 78.32{\scriptsize$\pm$0.00} & 78.87{\scriptsize$\pm$0.00} & 79.09{\scriptsize$\pm$0.01} & 79.17{\scriptsize$\pm$0.00} \\
Eq + Inv & 78.32{\scriptsize$\pm$0.01} & 78.94{\scriptsize$\pm$0.03} & 79.28{\scriptsize$\pm$0.01} & 79.43{\scriptsize$\pm$0.05}\\
\bottomrule
\end{tabular}
\end{sc}
\end{small}
\end{center}
\end{table}

\subsection{\mse captures meaningful diversity that leads to improved performance}
\label{sec:diversity_vs_performance}
We now compare deep ensembles (DE) constructed with models from the leading hypothesis ($\mathrm{Eq}$) against \mse, which combines models from both hypotheses ($\mathrm{Eq} + \mathrm{Inv}$), as shown in~\cref{tab:IN_results} for ImageNet. 
Intuitively, given that $\mathrm{Eq}$ outperforms $\mathrm{Inv}$ significantly by 0.9\%, one might expect to get larger gains by adding high-accuracy models from the leading hypothesis to the ensemble. 
Instead, we found ensembles involving lower-accuracy models from the opposing hypothesis to be better, with \mse ($\mathrm{Eq} + \mathrm{Inv}$) outperforming DE of rotational equivariant models ($\mathrm{Eq}$) consistently across all ensemble sizes.
\cref{fig:E_vs_EI_greedy} further highlights the gap between the ensemble accuracy of $\mathrm{Eq} + \mathrm{Inv}$ and $\mathrm{Eq}$.
Ensembles constructed only from the leading hypothesis quickly result in marginal improvements gained from adding more members; by $M=5$, the ensemble accuracy plateaus and does not benefit from further addition of more models. 
On the other hand, the ensemble accuracy of \mse demonstrates \textit{greater potential and continues to benefit from increasing ensemble sizes}.
\vspace{-0.3em}

\begin{figure}[t]
\begin{center}
\centerline{\includegraphics[width=\columnwidth]{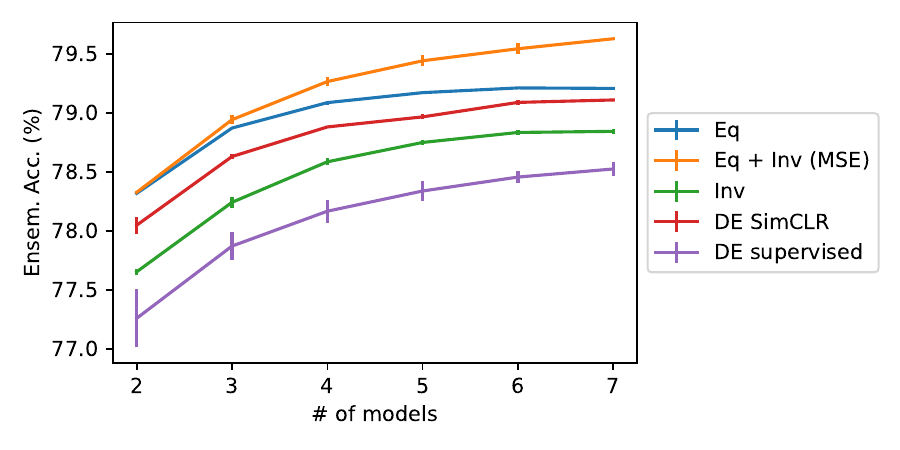}}
\caption{\textbf{Ensembles with opposing hypotheses have significantly larger potential.} Ensembles constructed only from a single hypothesis very quickly give marginal ensembling gains from adding more members. DE SimCLR and DE supervised refer to deep ensembles of baseline SimCLR (neither equivariant nor invariant) and supervised learning (without pretraining) 
respectively.}
\label{fig:E_vs_EI_greedy}
\end{center}
\vskip -0.3in
\end{figure}

\paragraph{Greedy search finds alternating sequences.}
Interestingly, the outcome of the greedy search produces the following sequence of models: [$\mathrm{Eq}$, $\mathrm{Inv}$, $\mathrm{Eq}$, $\mathrm{Eq}$, $\mathrm{Inv}$, $\mathrm{Eq}$, $\mathrm{Inv}$], that almost alternates between adding an equivariant and an invariant model at every step.
This result suggests that in order to best maximize ensemble accuracy, it is ideal to construct ensembles that contain opposing hypotheses. 
\vspace{-0.2em}

\paragraph{MSE's performance can be attributed to greater ensemble diversity.}
To further analyze the effectiveness of \mse ($\mathrm{Eq}+\mathrm{Inv}$) over the DE of $\mathrm{Eq}$ hypotheses, we evaluate their diversity on commonly used metrics, such as the error inconsistency~\cite{no_one_rep_rules} between pairs of models, variance in predictions~\cite{kendall_diversity} and pair-wise divergence measures~\cite{DE_losslandscape} of the prediction distribution.
We use error inconsistency as the main measure of diversity given its intuitive nature, which can be described as the fraction of samples where only one of the models makes the correct prediction, averaged over all possible pairs of models in the ensemble. 
Other diversity measures are defined in~\cref{appendix:div_measures}.
Ensemble diversity is an important criterion since higher ensembling performance is derived when individual models make mistakes on different samples. 
~\cref{tab:div} demonstrates that by including models from opposing hypotheses, \mse indeed achieves a greater amount of diversity compared to the DE of $\mathrm{Eq}$, consistently across all the diversity metrics.  

\paragraph{Comparison between ensembling methods.}
\cref{fig:div_eff} further compares \ourmodels across some alternative methods to creating ensembles: ensembling models trained with supervised learning~\cite{DE} (Sup), models that are separately fine-tuned with randomly initialized linear head but using the same pre-trained backbone (SSL\_FT), models trained with the baseline SimCLR~\cite{simclr} (SSL), models trained with Equivariant SSL~\cite{essl} (E\_SSL) and models with opposing equivariance (E+I\_SSL). 
Apart from E+I\_SSL, all other methods create models from a single hypothesis.
Unsurprisingly, SSL\_FT produces ensembles with particularly poor diversity due to the limited variance between members since they differ only in the initalization of the linear heads.
In general, the ensemble diversity is directly correlated with the ensemble efficiency (defined as the performance improvement relative to the mean accuracy of all the models in the ensemble~\cite{no_one_rep_rules}).
However, larger ensemble diversity does not necessarily lead to greater ensemble accuracy, since it is also important for the individual models to be high performing.
This is evident in ensembles of supervised models -- while they demonstrate high diversity and ensemble efficiency, their ensemble accuracy is poorer than their SSL counterparts since SSL produces higher performing models.
\begin{table}[tb]
\caption{\textbf{Diversity of ensembles.} We compare the diversity across several metrics for ensembles with $M=3$ members: error inconsistency, variance of the logits, variance of the probabilities and KL-divergence between pair-wise predictions. In all metrics, higher the score, greater the diversity. }
\label{tab:div}
\setlength\tabcolsep{4.5pt}
\begin{center}
\begin{sc}
\scalebox{0.85}{
\begin{tabular}{lcccccc}
\toprule
& Incons.(\%) & Logits & Prob ($10^{-4}$) & KL-div \\
\midrule
\rebutadd{Inv}  & \rebutadd{17.0{\scriptsize$\pm$0.1}} & \rebutadd{0.88{\scriptsize$\pm$0.02}} & \rebutadd{2.85{\scriptsize$\pm$0.04}} & \rebutadd{0.332{\scriptsize$\pm$0.012}} \\
Eq  & 15.6{\scriptsize$\pm$0.1} & 0.82{\scriptsize$\pm$0.01} & 2.64{\scriptsize$\pm$0.00} & 0.287{\scriptsize$\pm$0.001} \\
Eq + Inv & 17.5{\scriptsize$\pm$0.1} & 0.94{\scriptsize$\pm$0.01} & 2.94{\scriptsize$\pm$0.00} & 0.359{\scriptsize$\pm$0.007} \\
\bottomrule
\end{tabular}
}
\end{sc}
\end{center}
\vspace{-1em}
\end{table}

\begin{figure}[tb]
\vskip 0.2in
\begin{center}
\centerline{\includegraphics[width=\columnwidth]{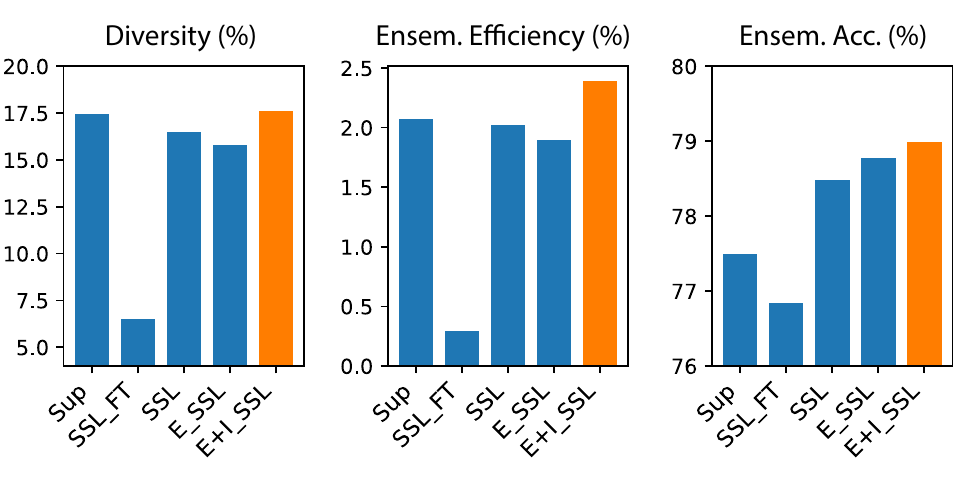}}
\vskip -0.1in
\caption{\textbf{Comparison between ensembling methods} for $M=3$: supervised ensembles (Sup), models created from separate fine-tuning on the same backbone (SSL\_FT), models pre-trained with SimCLR~\cite{simclr}, models pre-trained with equivariant-SimCLR~\cite{essl} (E\_SSL) and \ourmodels of opposing hypotheses (E+I\_SSL). Diversity is measured by pair-wise error inconsistency, ensemble efficiency is defined as the relative improvement over the mean accuracy of the members. }
\vskip -0.3in
\label{fig:div_eff}
\end{center}
\end{figure}

\subsection{\mse can quantify uncertainty better but may require more models}
\label{sec:uq}
A strong motivation to using an ensemble of models is it provides a way to quantify uncertainty from a Bayesian perspective~\cite{wilson2020case,uq_ovadia}. 
To evaluate the quality of the ensembles' uncertainty estimates, we use the negative log likelihood (NLL), 
which is a proper scoring rule and a popular metric used to evaluate \emph{predictive uncertainty}~\cite{DE,pred_uncertainty_challenge}.
As seen from~\cref{tab:uq}, for ensembles consisting of 3 members or more, \ourmodels built from opposing hypotheses (Eq + Inv) performs slightly better in terms of NLL when compared to ensembles with members only sampled from a single hypothesis (Eq).
However, when there are very few members in the ensemble ($M=2$), \ourmodels of opposing hypotheses performed slightly worse than single hypothesis according to the NLL metric. 
This is perhaps unsurprising, since the space of hypotheses is much larger with models from non-overlapping hypotheses and thus such an ensemble is likely to require more members in order to quantify the uncertainty surrounding each of these hypotheses.
The NLL results show consistent trends the different ensembles.

\begin{table}[tb!]
\caption{\textbf{Uncertainty Quantification.} We evaluate the uncertainty quantification of our greedy ensembles, using the negative log likelihood loss (NLL) and the `area under the uncertainty quantification curve' (AUUQC) which is obtained by sequentially removing the most uncertain samples and computing the area under the plot of ensemble accuracy versus fraction of samples removed. See Appendix \ref{appendix:uncertainty} for results on our random ensembles.}
\label{tab:uq}
\setlength\tabcolsep{3pt}
\begin{center}
\begin{small}
\begin{sc}
\begin{tabular}{lcccccccc}
\toprule
& $M=2$ & $M=3$& $M=4$ & $M=5$ \\
\midrule
\multicolumn{4}{l}{\textbf{NLL} $\downarrow (10^{-1})$} \\
\midrule
\rebutadd{Inv}  & \rebutadd{8.77{\tiny$\pm$0.03}} & \rebutadd{8.49{\tiny$\pm$0.01}} & \rebutadd{8.35{\tiny$\pm$0.00}} & \rebutadd{8.26{\tiny$\pm$0.00}}\\
Eq  & 8.40{\tiny$\pm$0.00} & 8.18{\tiny$\pm$0.00} & 8.06{\tiny$\pm$0.00} & 7.99{\tiny$\pm$0.00}\\
Eq + Inv & 8.43{\tiny$\pm$0.03} & 8.16{\tiny$\pm$0.01} & 8.03{\tiny$\pm$0.01} & 7.97{\tiny$\pm$0.01}\\
\midrule
\multicolumn{4}{l}{\textbf{AUUQC} $\uparrow$} \\
\midrule
\rebutadd{Inv}  & \rebutadd{0.919{\tiny$\pm$0.000}} & \rebutadd{0.924{\tiny$\pm$0.000}} & \rebutadd{0.925{\tiny$\pm$0.000}} & \rebutadd{0.926{\tiny$\pm$0.000}} \\
Eq  & 0.921{\tiny$\pm$0.000} & 0.925{\tiny$\pm$0.000} & 0.927{\tiny$\pm$0.000} & 0.928{\tiny$\pm$0.000} \\
Eq + Inv & 0.921{\tiny$\pm$0.000} & 0.926{\tiny$\pm$0.000} & 0.928{\tiny$\pm$0.000} & 0.929{\tiny$\pm$0.000} \\
\bottomrule
\end{tabular}
\end{sc}
\end{small}
\vspace{-1em}
\end{center}
\end{table}

To further evaluate the ensembles' ability to quantify \emph{model uncertainty} \citep{Gal2016Uncertainty}, we also consider a different metric using an uncertainty-based prediction rejection setup, described as follows.
We sequentially remove pools of test samples with the highest uncertainty from the ensemble and evaluate the ensemble accuracy on the remaining samples. 
This allows us to plot a curve of \textit{fraction of samples removed} against \textit{ensemble accuracy} which asymptotically approaches one when all samples are removed.  
An ensemble that ``knows when it does not know'' would produce a curve that is closer to the upper-left corner, since it can more accurately remove uncertain samples to give higher ensemble accuracies more quickly.
We use the commonly used uncertainty measure BALD in the active learning framework~\citep{yarin_uq,bald_uq},
which is defined the information gained of the model parameters;
see Appendix~\ref{appendix:bald} for a definition.
Samples with large BALD would have the highest probability assigned to a different class on every stochastic forward pass~\cite{yarin_uq} and thus have the highest model uncertainty.
We compute and report the area under this curve and call it the ``Area under the uncertainty quantification curve'' (AUUQC) --- higher AUUQC signifies better uncertainty quantification (see~\cref{fig:auuqc} in~\cref{appendix:auuqc} for an illustration of this curve).
Under this metric, we found that ensembles of opposing hypotheses (Eq + Inv) consistently outperforms ensembles of a single hypothesis across ensembles of different sizes.

\begin{table}[t]
\caption{\textbf{Ensemble performance on transfer tasks using the greedy approach.} Ensemble efficiency is defined as the relative improvement from the mean accuracy of all the models in the ensemble. All experiments are fine-tuned except iNaturalist-1k which is linear-probed. Note that by construct of the greedy approach, $\mathrm{Eq}+\mathrm{Inv}$ searches over possible $\mathrm{Eq}$ and $\mathrm{Inv}$ models and thus will be \textit{at least as good as $\mathrm{Eq}$}, i.e. datasets with equal performance for $\mathrm{Eq}$ and $\mathrm{Eq}+\mathrm{Inv}$ do not benefit from the opposing hypothesis. See ~\cref{appendix:transfer_random} for results on our random ensembles.}
\label{tab:transfer1}
\setlength\tabcolsep{3pt}
\vskip -0.3in
\begin{center}
\begin{sc}
\resizebox{0.5\textwidth}{!}{
\begin{tabular}{lcccc}
\toprule
         & iNaturalist-1K & Flowers-102 & CIFAR-100 & Food-101  \\ \midrule
\multicolumn{5}{l}{\textbf{Single Model Accuracy}}                            \\ \midrule
Eq       & 55.1 {\tiny$\pm$0.3}     & 91.9 {\tiny$\pm$0.0}      & 85.5 {\tiny$\pm$0.1}      & 87.9 {\tiny$\pm$0.1}              \\
Inv      & 56.3 {\tiny$\pm$0.2}     & 91.2 {\tiny$\pm$0.1}      & 84.0 {\tiny$\pm$0.1}      & 87.9 {\tiny$\pm$0.1}            \\ \midrule
\multicolumn{5}{l}{\textbf{Ensemble Accuracy} ($M=2$) \green{Ensemble efficiency}}                          \\ \midrule
Eq       & 58.4 {\tiny$\pm$0.0} \green{3.3} & 92.7 {\tiny$\pm$0.0} \green{0.8}        & 86.6 {\tiny$\pm$0.0} \green{1.1}     & 89.3 {\tiny$\pm$0.1} \green{1.4}         \\
Eq $+$ Inv & 59.9 {\tiny$\pm$0.2} \green{4.2} & 93.1 {\tiny$\pm$0.1} \green{1.5}       & 86.6 {\tiny$\pm$0.1} \green{1.4}      & 89.5 {\tiny$\pm$0.1} \green{1.6}           \\ \midrule
\multicolumn{5}{l}{\textbf{Ensemble Accuracy} ($M=3$) \green{Ensemble Efficiency}}                          \\ \midrule
Eq       & 59.8 {\tiny$\pm$0.0} \green{4.7} & 92.9 {\tiny$\pm$0.1} \green{1.0}      & 87.1 {\tiny$\pm$0.0} \green{1.6}      & 89.9 {\tiny$\pm$0.0} \green{2.0}           \\
Eq $+$ Inv & 61.4 {\tiny$\pm$0.1} \green{5.5} & 93.2 {\tiny$\pm$0.1} \green{1.4}      & 87.2 {\tiny$\pm$0.0} \green{1.8}      & 90.1 {\tiny$\pm$0.1} \green{2.2}           \\ \bottomrule
\end{tabular}
}
\end{sc}
\end{center}
\vskip -0.15in
\end{table}

\subsection{Different tasks may have different leading hypotheses and thus \mse transfers better}
Another important axis to evaluate is the generalization of the learned representations in \mse. To this end, we conduct transfer learning experiments using pre-trained \mse on four downstream tasks. As shown in  Table \ref{tab:transfer1}, \mse improves transfer performance in majority of the cases. In the largest and most diverse dataset iNaturalist-1K, we see consistent improvements of $1.5\%$ and $1.6\%$ from \mse in the respective cases of $M=2$ and $M=3$. Also, across the four transfer tasks, it is evident that ensemble efficiency, the change in performance of the ensemble relative to the mean accuracy of the individual models in the ensemble, always improves significantly with our method except in one case. In Section \ref{sec:effectiveness}, we further empirically analyze the circumstances under which our \ourmodels prove to be more useful. An interesting phenomenon to highlight in these results is that the dominant hypothesis can change depending on the downstream task. In the pre-training dataset (ImageNet), the equivariant model always proved to be the dominant hypothesis, outperforming the invariant model by $0.9\%$. However, after transfer learning on iNaturalist-1K, for example, the invariant model switches to become the dominant hypothesis, outperforming the equivariant model by $1.2\%$. This result emphasizes that different downstream tasks encompass different sets of hypotheses and therefore an ensemble of opposingly equivariant models can lead to better generalization.

\subsection{Effectiveness of \mse depends on dataset diversity}
\label{sec:effectiveness}

In this section, we aim to provide empirical guidance on when the inclusion of opposing hypotheses in an ensemble is beneficial. We evaluate the proportion of classes dominated by each of the opposing hypotheses (invariant and equivariant symmetries) for different datasets, including iNaturalist-1k, CIFAR-100, ImageNet-V2, and ImageNet-R. These results are shown in \cref{fig:eff_bar}. Our findings indicate that on datasets such as iNaturalist-1k, the inclusion of opposing hypotheses in the ensemble improves performance. However, on datasets like CIFAR-100 and ImageNet-R, the opposing hypotheses do not provide significant gains. This is because these datasets have a high level of imbalance between the dominance of the two hypotheses, with one hypothesis dominating in a majority of the classes. For example, in ImageNet-R, the equivariant hypothesis dominates in 76.5\% of the classes while the invariant hypothesis only dominates in 18\% of classes. These datasets are poorly described by the opposing hypothesis and thus including them in the ensemble provides little to no improvement in performanc\begin{figure}[t]
\begin{center}
\centerline{\includegraphics[width=\columnwidth]{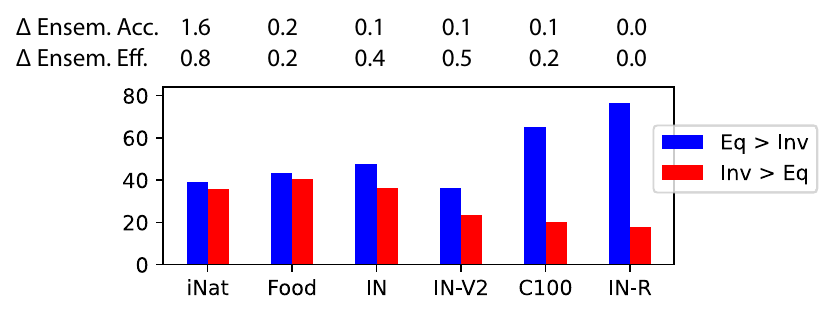}}
\caption{\textbf{Understanding the effectiveness of including the opposing hypothesis.} Plot shows the proportion of classes in each dataset where each hypothesis dominates. The remaining proportions (not shown) are classes where Eq and Inv are equally performant. Gains are minimal in datasets with a high level of imbalance between the leading and opposing hypothesis.}
\label{fig:eff_bar}
\end{center}
\vskip -0.2in
\end{figure}e.\looseness=-1

\subsection{Exploring different symmetry groups captures further meaningful diversity}
\label{sec:transform}
So far, we have shown that capturing opposing hypothesis along the axis of rotational symmetry increases the diversity and performance of model ensembles. 
A natural question arises: can other symmetry groups be useful as well? 
Specifically, referring back to the illustration in~\cref{fig:illus}, is it sufficient to capture diversity around the opposing hypotheses of a single symmetry group, or would opposing hypotheses across symmetry groups further add meaningful diversity? 
To address this question, we conduct an ablation study with two additional transformations, half swap (random swapping of the upper and lower halves of an image) and color inversion (randomly inverting the color of an image). 
Due to computational limitations, we conduct this ablation study on ImageNet-100, a subset of ImageNet generated by randomly selecting 100 classes from ImageNet-1k and train the classifiers using linear-probing. 
This dataset contains about 129k samples and thus is still sufficiently diverse.

We present the results in~\cref{tab:mixingtransform}.
In the upper three rows, we create ensembles that consist of both equivariant and invariant learners with respect to a single axis of transformation. In the last row, we greedily search over the space of models that were trained across the three axes of transformations (rotation, half swap, and color inversion). By exploring multiple symmetry groups, we find additional diversity that improves the performance by up to 1.2\%. 
This result bolsters the value of exploring multiple groups of opposing hypotheses and highlights the potential for future research directions to more effectively combining these models.

\begin{table}[t]
\caption{\textbf{Capturing opposing hypotheses across transformations for $M=6$.} The upper three rows are ensembles that consist of both equivariant and invariant learners with respect to a single transformation and the bottom row greedily searches over all models across the three transformations.}
\label{tab:mixingtransform}
\vskip 0.15in
\begin{center}
\begin{small}
\begin{sc}
\begin{tabular}{@{}lcc@{}}
\toprule
\multicolumn{3}{l}{\textbf{Ensemble Accuracy on ImageNet-100 (\%)}}                  \\ \midrule
Rotate                     & \multicolumn{2}{c}{86.60} \\
Halfswap                    & \multicolumn{2}{c}{86.00} \\ 
ColorInvert                    & \multicolumn{2}{c}{86.26} \\  \midrule
Rotate + Halfswap + ColorInvert                   & \multicolumn{2}{c}{87.22} \\ \bottomrule
\end{tabular}
\end{sc}
\end{small}
\end{center}
\vskip -0.15in
\end{table}

\section{Conclusion \rebutadd{and Limitations}}
In this work, we have showed that many large vision datasets benefit from a multiplicity of hypotheses, particularly along different axes of symmetries.
To address this, we proposed to ensemble members from opposing hypotheses, disregarding the fact that models from the opposing hypothesis are significantly poorer performing. 
We showed that despite their lower accuracies, ensembles containing the opposing hypotheses are meaningfully diverse and outperform current ensembling approaches of exploring the leading hypothesis class in multiple metrics of ensemble performance, ensemble potential, uncertainty quantification and generalization across transfer tasks.

While we explored a simple deep ensembling approach to combine multiple hypotheses, in principle one could also combine these hypotheses in alternative model combination approaches such as stacking~\cite{WOLPERT1992241} and mixture of experts~\cite{moe,moe2,moe_ensemble}.
\rebutadd{Furthermore, since equivariance and invariance are invoked in the pre-training stage, the construction of these ensembles have higher computational costs compared to supervised deep ensembles that are trained from scratch (but on par with deep ensembles of contrastive learners), further work could look into more efficient methods to invoke equivariance and invariance during fine-tuning to mitigate this.
Finally, while we found \mse to be highly effective in diverse, natural vision datasets, its effectiveness is dependent on dataset diversity (for e.g. less effective in ImageNet-R); we provide some intuition for these cases in Appendix B.
}
We hope the findings from our work can motivate future research in these directions.

\section{Acknowledgements}
This work was sponsored in part by the United States Air Force Research Laboratory and the United States Air Force Artificial Intelligence Accelerator and was accomplished under Cooperative Agreement Number FA8750-19-2-1000. The views and conclusions contained in this document are those of the authors and should not be interpreted as representing the official policies, either expressed or implied, of the United States Air Force or the U.S. Government. The U.S. Government is authorized to reproduce and distribute reprints for Government purposes notwithstanding any copyright notation herein. 

C.L. acknowledges fellowship support from the DSO National Laboratories, Singapore.

\newpage

\nocite{langley00}

\bibliography{example_paper}
\bibliographystyle{icml2023}

\newpage
\appendix
\onecolumn
\section{Formalism and Intuition}
\label{app:formalism}

We show analytically that the functional classes of invariant and equivariant contrastive learners are different in a simple setting. As our assumptions are strong and simplistic, we only aim to provide intuition through a simple formalism. 
Our experiments in~\cref{sec:results} support this intuition without the strong assumption and demonstrate the diversity from equivariance using real-world examples.

\begin{assumption}\label{assumpt:simple_analytical}
Consider a linear model class from~\citep{kumar2022fine}, which is $f_{v,B}(x)=v^TBx$, where $B \in \mathbb{R}^{k \times d},$ is a linear encoder, $v \in \mathbb{R}^k$ is a linear head and $x \in \mathbb{R}^d$ is a datapoint. Consider an invariant model, $f^\mathrm{inv}_v(x) \colonequals v^TBx,$ such that $BT_g(x)=Bx$ for every $g \in G$ and $x \in X$. Let $f^\mathrm{equiv}_v(x) \colonequals v^TB'x$ be an equivariant model, such that $B'T_g(x)=T'_g(B'x)$ for all $g \in G$.  Here we assume that $B$ and $B'$ are pretrained and fixed encoders and so we are only training $v$. Thus, we can represent $v \equiv v(B)$ as a function of the backbone $B$. Let $\tilde{X}=[X^T \mid T_g(X)^T]^T \in \mathbb{R}^{2n \times d}$ be our training input data for some $X = \{x_i\}_{i=1}^n \in \mathbb{R}^{n \times d}$, a fixed group element $g \in G$, and $T_g(X) \colonequals \{T_g(x_i)\}_{i=1}^n.$ Assume the labels are $\tilde{y} = [y^T | y'^T]^T \in \mathbb{R}^{2n \times 1}$ where $y$ are the labels for $X$ and $y'$ are the corresponding labels for $T_g(X).$ Here we assume that the data contains all input images from $X$ and their transformation by $T_g$. Finally, assume an ordinary least squares (OLS) problem for learning $v$ with $(\tilde{X}B^T,\tilde{y})$ training data for the invariant case and $(\tilde{X}B'^T,\tilde{y})$ for the equivariant.
\end{assumption}
\begin{proposition}\label{prop:simple_analytical}
Under Assumption~\ref{assumpt:simple_analytical},
the solutions $v^\mathrm{inv}$ and $v^\mathrm{equiv}$ to the ordinary least squares problem for the corresponding $f^\mathrm{inv}$ and $f^\mathrm{equiv}$ with $(\tilde{X}B^T,\tilde{y})$ training data for the invariant case and $(\tilde{X}B'^T,\tilde{y})$ for the equivariant are:
\begin{align*}
v^\mathrm{inv}(B) & = & \frac{1}{2}(BX^TXB^T)^{-1}BX^T(y + y') \\
v^\mathrm{equiv}(B') & = & (B'X^TXB'^T + T'_g(XB'^T)^TT_g'(XB'^T))^{-1} (BX^Ty+T_g'(XB'^T)^Ty').
\end{align*}
\end{proposition}
\begin{proof}
The proof is a simple combination of the OLS solution and the equivariance property. Namely, if the input data is $A$ and the target is $b$, then the OLS solution is $(A^TA)^{-1}A^Tb.$ Now, it suffices to replace the placeholder $b$ with $\tilde{y}$, and the placeholder $A$ with $\tilde{X}B'^T$ in the invariant case, and $\tilde{X}B'^T = [B'X^T | T_g'(XB'^T)^T]^T$ in the equivariant case. For the invariant case, we use the invariance property, which yields $T_g(X)B^T=XB^T$.
For the equivariant case, we use the equivariance property, which yields us $T_g(X)B'^T=T'_g(XB'^T).$ Simplifying the algebra completes the proof.
\end{proof}
As functions of the pretraining backbones ($B$ and $B'$), the two models in Assumption~\ref{assumpt:simple_analytical} yield different functional classes (or hypotheses) as it can be seen by the forms of solutions in Proposition~\ref{prop:simple_analytical}.
This analytical example provides us with further motivation to leverage on self-supervised models with opposing equivariances to capture diversity around multiple hypotheses.
We choose to ensemble these different members instead of training one model because a single model cannot be simulataneouly invariant and equivariant to the same transformation due to conflicting objectives (i.e., a model cannot be invariant to a transformation and still change its representations according to the transformation). 

\rebutadd{
\section{More discussion towards Equivariance and Invariance}
}
\label{appendix:more_discussion}
\rebutadd{
\paragraph{Comparison with Group Equivariant networks.}
The general notion of "Equivariance" typically emcompasses both non-trivial equivariance and invariance (i.e. trivial equivariance where $T_g'$ of~\cref{eq:eq_loss} is the identity.
For brevity and to maintain the convention used in~\cite{essl}, we however use the term ``equivariance" to specifically refer only to non-trivial equivariance (i.e. excluding invariance) in our work.
Equivariance in deep learning is most commonly known through the concept of Group Equivariant neural networks~\cite{cohen2016group,weiler_general_2019,weiler_3d_2018}. 
There, non-trivial equivariance and invariance to a particular group are achieved through equivariant architectures, by generalizing convolutional kernels to respect the symmetries of that group. 
These are often implemented in the form of equivariant layers, where the trivial instance of invariance can be acheived by invoking a global pooling function after a series of equivariant layers.
In our work, (non-trivial) equivariance and invariance to a particular transformation $T_m$ are achieved purely via training objectives --- invariance is achieved by adding $T_m$ into the set of augmentations used in contrastive learning that encourages representations to be invariant to and equivariance is achieved by adding an auxiliary self-supervised task that predicts the transformation $T_m$ applied to the input.  
The architecture we use for all models is a non-equivariant architecture, i.e. the common ResNet-50 model.
In this setting and our definition of ``equivariance'' that refers only to non-trivial equivariance, a single model cannot be equivariant and invariant simultaneously and thus the two form a set of opposing hypotheses.
}
\rebutadd{
\paragraph{Empirical intuition.} 
Equivariance to rotation has been known to be highly beneficial for learning visual representations~\cite{rotnet,essl}, however the underlying reasons are not so clear.
Empirically, we found the usefulness of rotation equivariance is generally related to pose or the existence of rotational symmetry in the dataset. 
We found that rotation equivariance is useful in image classes that often occur with a clear stance, for e.g. some classes of animals, where an upside-down dog is almost never observed in the dataset and thus the ability to recognize the rotation would require the features to encode information about its pose~\cite{rotnet}, aiding the characterization of dogs. 
On the other hand, we found that rotation invariance is useful in image classes that do not occur with a clear stance (for e.g. corkscrews that can be pictured in any orientation) or in images that have a clear rotation symmetry (e.g. flowers imaged from the front or analog clocks). 
}
\rebutadd{
\paragraph{Empirical intuition on datasets where MSE are effective.}
In our work, we found the effectiveness of MSE to be highly dependent on dataset diversity.
In particular, if the datasets are poorly described by the opposing hypothesis (i.e. ImageNet-R) as discussed in section 5.4, the gains from MSE would be negligible. 
Here, we provide some intuition on why this may be so.
Following the intuition provided in the previous paragraph, we conjecture that this could be related to the existence of a dominant pose of images in the dataset. 
An example of the class of “jellyfish” in ImageNet (IN) and IN-R is shown in~\cref{fig:jellyfish}. 
In IN-R which contains renditions of the images, such as in cartoon and art, many images assume a conventional ``upright" pose of the jellyfish with its head on top and its tentacles trailing below vertically.
However, in IN where real-life jellyfish are imaged, they often occur in multiple poses. 
We believe this is true for other classes as well, since artists often draw objects in their ‘conventional pose’.
Thus, for IN, invariant models are useful for 36.3\% (v.s. equivariant models being useful in 47.7\%). 
In contrast, for IN-R, invariant models are dominant only for 18\% of the classes (v.s. equivariant models being dominant in 76.5\%). Given the existence of an upright pose in IN-R, equivariant models that encode pose information are likely more useful than invariant models leading to this stark difference.  
}

\begin{figure}[ht]
\begin{center}
\includegraphics[width=0.44\columnwidth]{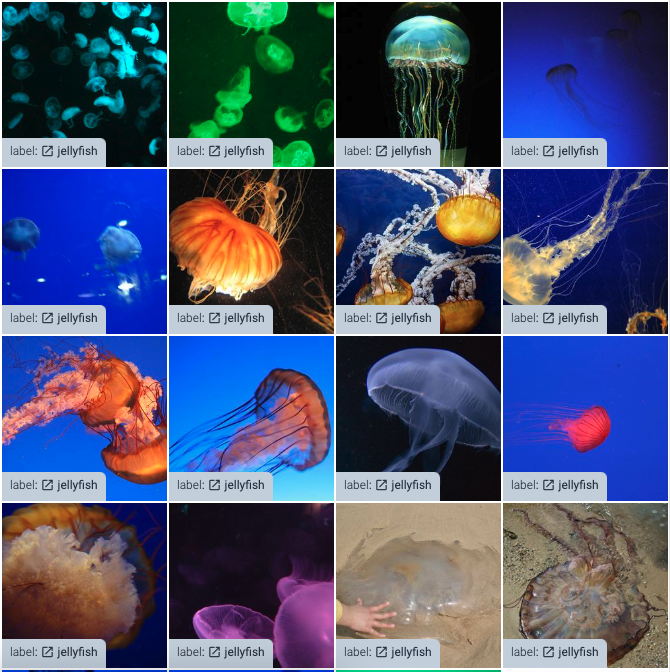}
\;\;\;\;
\includegraphics[width=0.45\columnwidth]{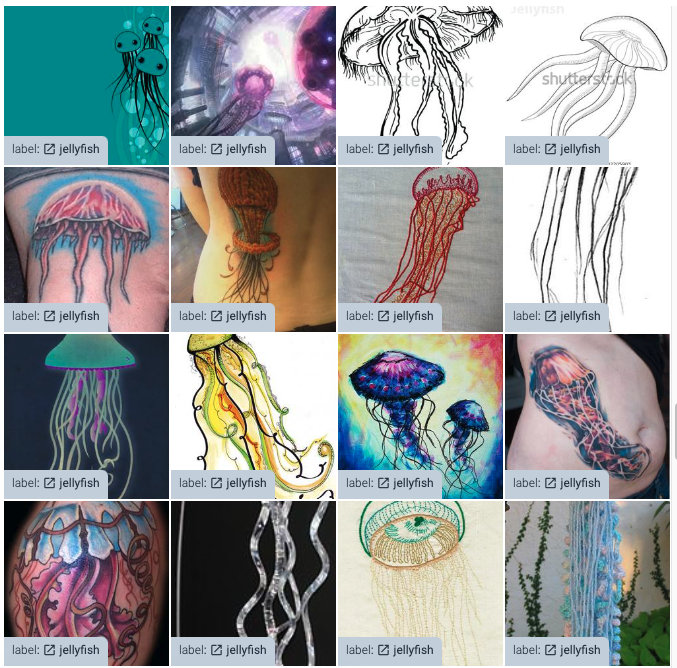}
\caption{\textbf{Examples of images from the ``jellyfish'' class in ImageNet (left) and ImageNet-R (right).} Samples visualized using \url{https://knowyourdata-tfds.withgoogle.com/}}
\label{fig:jellyfish}
\end{center}
\end{figure}

\rebutadd{
\section{Additional Training Details}
}\label{appendix:training_details}
\rebutadd{
\paragraph{All pre-training.}
We use the SGD optimizer with a learning rate of 4.8 (0.3 $\times BatchSize/256$). 
We decay the learning rate with a cosine decay schedule without restarts. 
Following~\cite{essl}, $T_{base}$ uses a slightly more optimal implementation that uses BYOL's augmentation (i.e. including solarization).
\paragraph{Equivariant pre-training.}
Following~\cite{essl}, the predictor for equivariance uses a smaller crop of $96\times 96$.
The predictor network uses a 3-layer MLP with a hidden dimension of 2048 to predict the corresponding transformation (i.e. 4-way rotation).
\paragraph{Invariant pre-training.}
For invariant models, the transformation $T_m$ is added to the base set of augmentations $T_{base}$ with probability $p=0.5$, i.e. with $0.5$ probability, one of the possible transformations ($0^{\circ},90^{\circ},180^{\circ},270^{\circ}$ for the case of 4-fold rotations) are applied. 
\paragraph{Explored hyperparameters for fine-tuning.}
For fine-tuning on ImageNet, we swept the learning rate ($lr \in \{0.1,0.03,0.01,0.003,0.004 \}$ for both equivariant and invariant models. 
We found $lr= 0.003$ to consistently give the best performance for equivariant models and $lr=0.004$ to consistently give the best performance for invariant models.
For fine-tuning on transfer tasks, we swept the learning rate $lr \in \{0.003, 0.1, 0.2, 0.5, 1.0, 5.0\}$ for each equivariant/invariant model and picked the best learning rate.
We set the weight-decay to $10^{-6}$ for all fine-tuning experiments.
}

\section{Ensemble Diversity}
\subsection{Diversity measures}
\label{appendix:div_measures}
\paragraph{Error inconsistency.} Following~\cite{no_one_rep_rules}, we use error inconsistency between pairs of models to quantify their diversity. 
For every sample and a pair of models, model A and model B, there are four possibilities: 1) both models are correct, 2) both models are wrong, 3) model A is correct and model B is wrong and 4) model B is correct and model A is wrong.
Samples that fall into the case of (3) and (4) constitute to the error inconsistency. 
We report the percentage of samples in the test set that pairs of models make inconsistent errors on. 
For ensembles more than $M=2$ members, we take the average over all possible pairs of models.

\paragraph{Variance of predictions.}
Another measure one could use to measure ensemble diversity is the variance of the predictions~\cite{kendall_diversity}:
\begin{equation}
\mathrm{Var}_{p(\mathbf{f})}[\mathbf{f}(\mathbf{x})] = \sum_{i=1}^C \mathrm{Var}_{p(\mathbf{f})}[f^{(i)}(\mathbf{x})]
\end{equation}
where $f^{(i)}$ refers to the probability assigned by the model to the $i$th class and $C$ is the total number of classes.
We report both the variance of the probabilities (labeled `prob' in~\cref{tab:div}) and the variance of the logits (before the softmax, labeled `logits' in~\cref{tab:div}).
\paragraph{Divergence measures.}
One can also use divergence metrics to quantify ensemble diversity~\cite{DE_losslandscape}. 
We simple use the KL-divergence between the prediction probability distributions of a pair of models, and take the average over all possible pairs in the ensemble.
\FloatBarrier

\subsection{Visualization of diversity across selected classes}
\label{appendix:vis_bars}
\Cref{fig:ei_vs_ee_bars} shows the accuracy per class for 10 randomly selected classes in ImageNet. 
The figure compares the performance of models trained with opposing equivariances (upper plot) and those with different random initializations (lower plot) and shows larger variances induced from opposing equivariant hypotheses.
Further analysis of their diversity is presented in~\cref{sec:diversity_vs_performance}.
The above results motivate the use of leveraging opposing equivariances as a method to induce diversity especially for large datasets like ImageNet. 

\begin{figure}[h!]
\vskip 0.2in
\begin{center}
\centerline{\includegraphics[width=0.6\columnwidth]{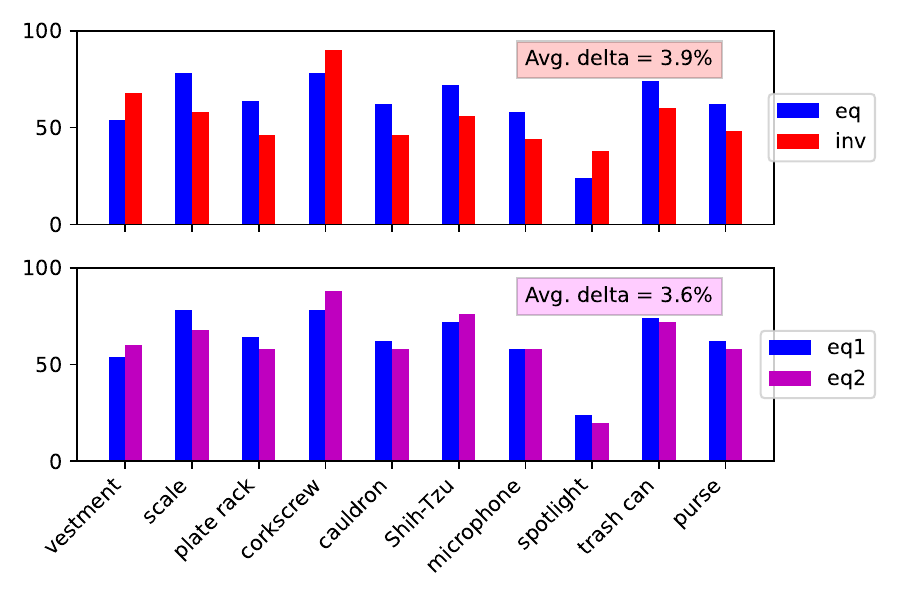}}
\caption{\textbf{Accuracy per class for 10 randomly selected classes in ImageNet.} Top panel compares the per class accuracy for a rotation equivariant model versus an invariant model and bottom panel compares the per class accuracy for two rotation equivariant models. }
\label{fig:ei_vs_ee_bars}
\end{center}
\vskip -0.2in
\end{figure}

\section{Uncertainty quantification results using random ensembles}

\cref{appendix:uncertainty} supplements the results in ~\cref{tab:uq} in the main text. While~\cref{tab:uq} shows the results for the greedy ensembling approach, this table shows the results for the random ensembling approach. In both of the cases, we see general improvements in the uncertainty quantification metrics with additional models.

\begin{table}[h!]
\caption{\textbf{Uncertainty Quantification.} We evaluate the uncertainty quantification of the ensembles using the negative log likelihood loss (NLL) and the `area under the uncertainty quantification curve' (AUUQC) which is obtained by sequentially removing the most uncertain samples and computing the area under the plot of ensemble accuracy versus fraction of samples removed.}
\label{tab:appendix_uq}
\setlength\tabcolsep{3pt}
\begin{center}
\begin{small}
\begin{sc}
\begin{tabular}{lcccccccc}
\toprule
& $M=2$ & $M=3$& $M=4$ & $M=5$ \\
\midrule
\multicolumn{5}{l}{\textbf{Random Ensemble} } \\
\midrule
\multicolumn{4}{l}{\textbf{NLL} $\downarrow (10^{-1})$} \\
\midrule
Eq  & 8.43{\scriptsize$\pm$0.03} & 8.20{\scriptsize$\pm$0.03} & 8.09{\scriptsize$\pm$0.02} & 8.03{\scriptsize$\pm$0.02}\\
Eq + Inv & 8.46{\scriptsize$\pm$0.01} & 8.17{\scriptsize$\pm$0.01} & 8.08{\scriptsize$\pm$0.02} & 7.98{\scriptsize$\pm$0.02}\\
\midrule
\multicolumn{4}{l}{\textbf{AUUQC} $\uparrow$} \\
\midrule
Eq  & 0.919{\tiny$\pm$0.001} & 0.924{\tiny$\pm$0.001} & 0.926{\tiny$\pm$0.000} & 0.927{\tiny$\pm$0.000}\\
Eq + Inv & 0.921{\tiny$\pm$0.001} & 0.926{\tiny$\pm$0.001} & 0.927{\tiny$\pm$0.000} & 0.928{\tiny$\pm$0.000}\\
\midrule
\multicolumn{5}{l}{\textbf{Greedy Ensemble} } \\
\midrule
\multicolumn{4}{l}{\textbf{NLL} $\downarrow (10^{-1})$} \\
\midrule
Eq  & 8.40{\tiny$\pm$0.00} & 8.18{\tiny$\pm$0.00} & 8.06{\tiny$\pm$0.00} & 7.99{\tiny$\pm$0.00}\\
Eq + Inv & 8.43{\tiny$\pm$0.03} & 8.16{\tiny$\pm$0.01} & 8.03{\tiny$\pm$0.01} & 7.97{\tiny$\pm$0.01}\\
\midrule
\multicolumn{4}{l}{\textbf{AUUQC} $\uparrow$} \\
\midrule
Eq  & 0.921{\tiny$\pm$0.000} & 0.925{\tiny$\pm$0.000} & 0.927{\tiny$\pm$0.000} & 0.928{\tiny$\pm$0.000} \\
Eq + Inv & 0.921{\tiny$\pm$0.000} & 0.926{\tiny$\pm$0.000} & 0.928{\tiny$\pm$0.000} & 0.929{\tiny$\pm$0.000} \\
\bottomrule
\end{tabular}
\end{sc}
\end{small}
\end{center}
\end{table}

\FloatBarrier
\section{Transfer results using random ensembles}
\label{appendix:transfer_random}
\Cref{tab:transfer_random} supplements the results in~\cref{tab:transfer1} in the main text. While~\cref{tab:transfer1} shows the results for the greedy ensembling approach, this table shows the results for the random ensembling approach.
\begin{table*}[h!]
\caption{\textbf{Transfer tasks for Random ensembles.} Ensemble efficiency is defined as the relative improvement over the mean accuracy of all the models in the ensemble. All experiments are fine-tuned except for iNaturalist-1k which is linear-probed.}
\label{tab:transfer_random}
\vskip 0.15in
\begin{center}
\begin{sc}
\resizebox{0.6\textwidth}{!}{
\begin{tabular}{@{}llllll@{}}
\toprule
         & iNaturalist-1K & Flowers-102 & CIFAR-100 & Food-101  \\ \midrule
\multicolumn{5}{l}{\textbf{Single Model Accuracy}}                            \\ \midrule
Eq       & 55.1 {\tiny$\pm$0.3}     & 91.9 {\tiny$\pm$0.0}      & 85.5 {\tiny$\pm$0.1}      & 87.9 {\tiny$\pm$0.1}              \\
Inv      & 56.3 {\tiny$\pm$0.2}     & 91.2 {\tiny$\pm$0.1}      & 84.0 {\tiny$\pm$0.1}      & 87.9 {\tiny$\pm$0.1}              \\ \midrule
\multicolumn{5}{l}{\textbf{Ensemble Accuracy} ($M=2$) \green{Ensemble efficiency}}                          \\ \midrule
Eq       & 58.3 {\tiny$\pm$0.1} \green{3.2}      & 92.3 {\tiny$\pm$0.4} \green{0.4}        & 86.6 {\tiny$\pm$0.2} \green{1.1}     & 89.2 {\tiny$\pm$0.1} \green{1.2}          \\
Eq $+$ Inv & 60.0 {\tiny$\pm$0.0} \green{4.3}      & 92.8 {\tiny$\pm$0.1} \green{1.3}       & 86.5 {\tiny$\pm$0.1} \green{1.8}      & 89.5 {\tiny$\pm$0.1} \green{1.6}        \\ \midrule
\multicolumn{5}{l}{\textbf{Ensemble Accuracy} ($M=3$) \green{Ensemble Efficiency}}                          \\ \midrule
Eq       & 59.8 {\tiny$\pm$0.0} \green{4.7}     & 92.4 {\tiny$\pm$0.2} \green{0.5}      & 87.1 {\tiny$\pm$0.1} \green{1.6}      & 89.9 {\tiny$\pm$0.0} \green{1.3}             \\
Eq $+$ Inv & 61.2 {\tiny$\pm$0.1} \green{5.5}       & 93.0 {\tiny$\pm$0.3} \green{1.3}      & 87.0 {\tiny$\pm$0.1} \green{2.0}      & 90.0 {\tiny$\pm$0.0} \green{2.1}        \\ \bottomrule
\end{tabular}
}
\end{sc}
\end{center}
\end{table*}
\FloatBarrier
\section{Uncertainty Quantification}\label{appendix:uncertainty}

\subsection{Definition of BALD}\label{appendix:bald}
In Section~\ref{sec:uq}, we use the commonly used uncertainty measure BALD~\citep{yarin_uq,bald_uq} to measure model uncertainty.
It is defined as below
\begin{equation*}
\mathbb{I}[y,\mathbf{w}|\mathbf{x}, \mathcal{D}] = \mathbb{H}[y|\mathbf{x}, \mathcal{D}] - \mathbb{E}_{p(\mathbf{w}|\mathcal{D})}\left[\mathbb{H}[y|\mathbf{x},\mathbf{w}]\right]
\end{equation*}
where $\mathcal{D}$ refers to the training set, $p(\mathbf{w}|\mathcal{D})$ is the posterior our ensemble approximates, $\mathbf{w}$ are the model parameters, i.e.~a member sampled from $p(\mathbf{w}|\mathcal{D})$, $\mathbb{H}[y|\mathbf{x}, \mathbf{w}]$ is the predictive entropy given model weights $\mathbf{w}$ and $\mathbb{H}[y|\mathbf{x},\mathcal{D}] = -\sum_c p(y=c|\mathbf{x},\mathcal{D}) \log p(y=c|\mathbf{x},\mathcal{D})$ is the entropy of the ensemble's prediction. 

\subsection{Area under uncertainty quantification curve (AUUQC)}
\label{appendix:auuqc}
\Cref{fig:auuqc} provides an illustration of the `uncertainty quantification curve' described in~\cref{sec:uq}, for ensembles of the leading hypothesis (rotation equivariant) with different ensemble sizes. 
As the ensemble size grows, the AUUQC increases as expected since a larger ensemble should be able to quantify uncertainty better.
\begin{figure}[h!]
\vskip 0.2in
\begin{center}
\centerline{\includegraphics[width=0.55\columnwidth]{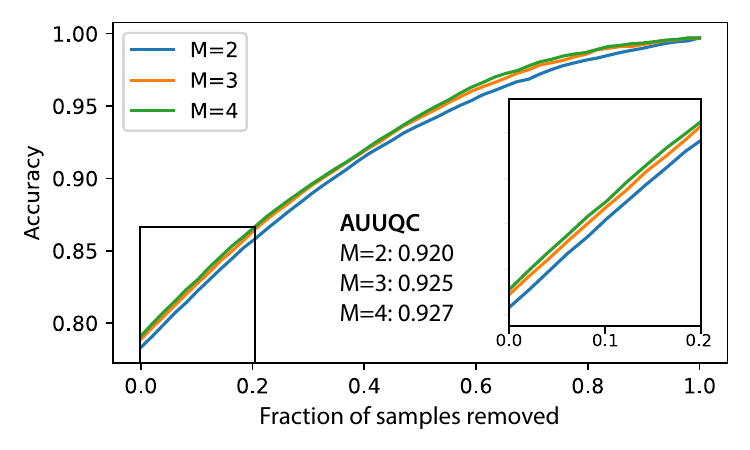}}
\caption{\textbf{Example of plot of the `uncertainty quantification curve' used to generate AUUQC}.}
\label{fig:auuqc}
\end{center}
\vskip -0.2in
\end{figure}

\section{Proportion of classes for hypothesis dominance}

\begin{table}[h!]
\caption{\textbf{Proportion of classes and performance gains in Transfer datasets.} The top half of the table detail the proportion of classes captured by dominating hypothesis for each transfer dataset. The bottom half describes the accuracy and ensemble efficiency gained by capturing opposing hypothesis over a single hypothesis. This table is used to generate the bar plot (\cref{fig:eff_bar}) in the main text}
\label{tab:transferprop}
\vskip 0.15in
\begin{center}
\begin{sc}
\scalebox{0.9}{
\begin{tabular}{lcccccc}
\toprule
& IN & iNat & Food & C100 & IN-V2 & IN-R \\
\midrule
Eq $>$ Inv  & 47.7 & 38.9 &  43.6 &  65.0 & 36.1 & 76.5 \\
Eq $<$ Inv  & 36.3 & 35.7 &  40.6 & 20.0 & 23.7 & 18.0 \\
\midrule
$\Delta_\mathrm{acc}$ (EI - EE) & +0.1 & +1.6 & +0.2 & +0.1 & +0.1 & 0.0 \\
$\Delta_\mathrm{eff}$ (EI - EE) & +0.4 & +0.8 & +0.2 & +0.2 & +0.5 & 0.0 \\
\bottomrule
\end{tabular}
}
\end{sc}
\end{center}
\end{table}

\begin{table}[h!]
\caption{\textbf{Proportion of classes and performance gains in ImageNet-100.} The top half of the table detail the proportion of classes captured by dominating hypothesis over different axes of transformations. The bottom half describes the accuracy and ensemble efficiency gained by capturing opposing hypothesis over a single hypothesis. This table supplements the results from \cref{tab:mixingtransform} in the main text}
\label{tab:in100proportions}
\vskip 0.15in
\begin{center}
\begin{sc}
\scalebox{0.9}{
\begin{tabular}{lcccccc}
\toprule
& Rotate & Halfswap & ColorInvert \\
\midrule
Eq $>$ Inv  & 65.0 & 28.0 & 48.0   \\
Eq $<$ Inv  & 15.0 & 44.0 & 26.0   \\
Eq $==$ Inv  & 20.0 & 28.07 & 26.0   \\
\midrule
$\Delta_\mathrm{acc}$ (EI - EE) &  0.0 & 0.04 & 0.14 \\
$\Delta_\mathrm{eff}$ (EI - EE) & 0.0 & 0.27 & 0.24 \\
\bottomrule
\end{tabular}
}
\end{sc}
\end{center}
\end{table}


\end{document}